\documentclass{article}

\usepackage{arxiv}

\usepackage[utf8]{inputenc} 
\usepackage[T1]{fontenc}    
\usepackage{hyperref}       
\usepackage{url}            
\usepackage{booktabs}       
\usepackage{amsfonts}       
\usepackage{nicefrac}       
\usepackage{microtype}      
\usepackage{graphicx}
\usepackage{natbib}
\usepackage{doi}

\newcommand\independent{\protect\mathpalette{\protect\independenT}{\perp}}
\def\independenT#1#2{\mathrel{\rlap{$#1#2$}\mkern2mu{#1#2}}}

\usepackage{amsmath}
\usepackage{amssymb}
\usepackage{mathtools}
\usepackage{amsthm}
\usepackage{float}
\usepackage{epsfig}
\usepackage{subfloat}
\usepackage{subfig} 
\usepackage{caption}

\theoremstyle{plain}
\newtheorem{theorem}{Theorem}[section]

\theoremstyle{definition}
\newtheorem{definition}[theorem]{Definition}

\theoremstyle{remark}

\title{Generative Causal Representation Learning for Out-of-Distribution Motion Forecasting}

\author{Shayan Shirahmad Gale Bagi\\
	Department of Electrical and Computer Engineering\\
	University of Waterloo\\
	  Waterloo, ON  N2L 3G1  \\
	\texttt{sshirahm@uwaterloo.ca} \\
	\And
	Zahra Gharaee \\
	Department of Systems Design Engineering\\
	University of Waterloo\\
	Waterloo, ON  N2L 3G1 \\
	\texttt{zahra.gharaee@uwaterloo.ca} \\
 	\And
	Oliver Schulte \\
	School of Computing Science\\
	Simon Fraser University\\
	Burnaby, B.C. V5A 1S6 \\
	\texttt{oschulte@cs.sfu.ca} \\
	\And
	Mark Crowley \\
	Department of Electrical and Computer Engineering\\
	University of Waterloo\\
	Waterloo, ON  N2L 3G1 \\
	\texttt{mcrowley@uwaterloo.ca} \\
}

\date{}

\renewcommand{\headeright}{}
\renewcommand{\undertitle}{}
\renewcommand{\shorttitle}{GCRL}

\hypersetup{
pdftitle={GCRL preprint},
pdfsubject={q-bio.NC, q-bio.QM},
pdfauthor={Shayan Shirahmad},
pdfkeywords={Motion forecasting, Causality, Domain Adaptation},
}

\begin{document}
\maketitle

\begin{abstract}
Conventional supervised learning methods typically assume i.i.d samples and are found to be sensitive to out-of-distribution (OOD) data. We propose Generative Causal Representation Learning (GCRL) which leverages causality to facilitate knowledge transfer under distribution shifts. While we evaluate the effectiveness of our proposed method in human trajectory prediction models, GCRL can be applied to other domains as well. First, we propose a novel causal model that explains the generative factors in motion forecasting datasets using features that are common across all environments and with features that are specific to each environment. Selection variables are used to determine which parts of the model can be directly transferred to a new environment without fine-tuning. Second, we propose an end-to-end variational learning paradigm to learn the causal mechanisms that generate observations from features. GCRL is supported by strong theoretical results that imply identifiability of the causal model under certain assumptions. Experimental results on synthetic and real-world motion forecasting datasets show the robustness and effectiveness of our proposed method for knowledge transfer under zero-shot and low-shot settings by substantially outperforming the prior motion forecasting models on out-of-distribution prediction.
Our code is available at \href{https://github.com/sshirahmad/GCRL}{https://github.com/sshirahmad/GCRL}. 
\end{abstract}


\section{Introduction}
Human Trajectory Prediction (HTP) is a valuable task in many applications such as infrastructure design, traffic operations, crowd abnormality detection systems, evacuation situation analysis, and autonomous vehicles \cite{motionsurvey}. It is also more challenging compared to other types of agent trajectory prediction, such as vehicles, due to the complex behaviour of humans in different environments. 
Most of the proposed models in the literature of motion forecasting \cite{SocialGAN, Social-BiGAT, PECNet, DESIRE, STGAT, STGCNN, Trajectron++} rely on statistical inference, which has two important shortcomings. First, the learnt representations are not explainable and, therefore, are ineffective to transfer into a new environment.
Secondly, statistical inference is susceptible to learn spurious correlations, which can significantly degrade the performance of the models in presence of a domain shift or measurement noise \cite{causalmotion}. In this work we aim to tackle these shortcomings from a causality perspective.

Causal representation learning has attracted much attention in various applications \cite{csg, CONTA, causal_survey, causalmotion}. Causal inference eliminates the confounding bias in predictions, \cite{causaljudea} which is a common phenomenon in motion forecasting since inputs and targets come from the same interdependent time series. 
Furthermore, models based on causal inference learn meaningful features that do not rely on spurious correlations. The coexistence of target trajectories and observation noise, is an example of spurious correlation in motion forecasting \cite{causalmotion}. Causality can also identify the generative factors in the dataset \cite{DisentangleR}. Physical laws, social norms, and motion styles are examples of generative factors in HTP. Under covariate shifts, a model that learns the causal structure of the data would require many fewer samples to adapt to the new environment because most of the modules can be reused without further training. This expectation is consistent with the Sparse Mechanism Shift hypothesis in causal representation learning \cite{causal_survey}.

To this end, we propose Generative Causal Representation Learning (GCRL) that leverages causality to increase the identifiability and robustness of current motion forecasting models. While our target application is motion forecasting, our proposed solution can be applied to domain adaptation in other applications. Oftentime, the training dataset is collected from different locations such as the ETH-UCY dataset \cite{eth-ucy1, eth-ucy2}. Considering this, we introduce a causal model that includes the invariant features, which are common across all types of environments, and variant features, which are environment-specific. Our cuasal model includes selection variables \cite{Transportability} to directly transfer some parts of the model to a new domain without fine-tuning. Second, we propose a new learning paradigm that can learn invariant and variant features simultaneously and eliminate the confounding bias using a backdoor adjustment
\cite{causaljudea}. Previous work \cite{causalmotion} uses Invariant Risk Minimization (IRM), \cite{irm} which requires solving a very complex minimization problem. 
Third, our proposed method uses a generative process to infer latent representations, which can capture multi-modal trajectory distributions, without adding any noise to the representations as proposed in \cite{STGAT}.

\textbf{Contributions} Our key contributions can be summarized as follows:
\vspace{-1em}
    \paragraph{(1)}A learning paradigm which enables end-to-end training and eliminating confounding bias.
\vspace{-0.8em}
    \paragraph{(2)} Augmented causally factorized model to enable direct transportability and reusability of some parts of the model.
\vspace{-0.8em}
    \paragraph{(3)} A generative process to produce latent representations to tackle multi-modality of trajectories.

\section{Related Work}
\vspace{-0.6em}
\subsection{Causal Representation Learning}
\vspace{-0.6em}
As pointed out in \cite{prop}, a transfer assumption is required in domain adaptation. Most of the methods in computer vision applications have parametric distribution families, hence, it is assumed that domain shifts affect the parameters of the distributions. Similar to our approach, INFER \cite{infer} uses an augmented version of Directed Acyclic Graphs (DAGs) to represent the casual structure of data. Their method assumes that the targets are not available in the test domain and perform unsupervised domain adaptation (DA) using a conditional GAN. CAE \cite{CAE} proposes and end-to-end framework to jointly learn a model for causal structure learning and a deep autoencoder model. Markov Blanket features \cite{markovblanket} are then extracted from the learnt causal model to be used in the test domain. CSG \cite{csg} leverages the Independent Causal Mechanisms (ICM) principle \cite{causal_survey} to learn the causal model for image classification task. In CSG, covariate shift is assumed to manifest itself in the priors of the latent variables in the form of spurious correlations.

CONTA \cite{CONTA} tackles the confounding bias in causal models by using the backdoor criterion \cite{causaljudea} for image segmentation task. To obtain invariant conditional distributions applicable in the test domain \cite{InvariantCD} uses the selection diagrams. CGNN \cite{CGNN}, on the other hand, applies maximum mean discrepancy between the generated distribution by a set of generative models and true distribution of the data to approximate the causal model of the data. DAG-GNN \cite{DAG-GNN} uses a variational autoencoder to learn the adjacency matrix of the DAG that is supposed to represent the causal model of the data. In \cite{unknownI}, the causal structure is learnt from both observational data and unknown interventions. In \cite{DisentangleR}, causal disentanglement is studied and a set of metrics are introduced to measure the robustness of the learnt features. 

Most of these methods attempt to learn the causal model of the data, however, in motion forecasting we can hypothesize a causal model based on the domain knowledge. Furthermore, identifying true causal model of the data is extremely challenging and impossible in an unsupervised setting with observational data \cite{impossibledisentangle}. Hence, the previous methods have not investigated causal discovery in real-world time-series data, which has exclusive challenges.

\vspace{-0.6em}
\subsection{Motion Forecasting}
\vspace{-0.6em}
Most of the deep-learning based motion forecasting models consist of an Encoder, an Interaction module, and a Decoder, \cite{motionsurvey} which we will refer to as EID architecture. STGAT \cite{STGAT} uses a graph attention network to model the human-human interactions and LSTM cells to encode past trajectories and to predict future trajectories. On the other hand, STGCNN \cite{STGCNN} consist of CNN layers only which significantly improves inference speed while performing on par with its RNN counterparts in terms of prediction error. 
Social-GAN \cite{SocialGAN} focuses on the multi-modality of human trajectories and uses GANs to generate multiple future trajectories per sample. Social-BiGAT \cite{Social-BiGAT} models both social and physical interactions using images of the environment and trajectories of pedestrians. DESIRE \cite{DESIRE} proposes a conditional variational autoencoder (CVAE) to tackle the multi-modality of human trajectories. Trajectron++ \cite{Trajectron++} focuses on the dynamic constraints of agents when predicting trajectories and models not only the interactions between humans but also interactions of humans with other agents in the scene. PECNet \cite{PECNet} uses a CVAE to obtain a distribution for the endpoints of the pedestrians, which are employed to predict the future trajectories conditioned on the endpoints. 
According to the recent works \cite{causalmotion, counterfactualmotion}, there are varying types of noise and intrinsic biases in the data resulting from human behaviour in different environments. The performance of the current motion forecasting models is negatively affected when transferred to a new environment or in the presence of a noise \cite{causalmotion}. Furthermore, there are some confounding factors in motion forecasting applications, which can lead to biased predictions \cite{counterfactualmotion}. Our work focuses on the robustness aspect of the motion forecasting problem. We propose a method applicable to motion forecasting models with an EID architecture, which will be discussed in the following sections.

\vspace{-0.8em}
\section{Proposed Method}\label{sec:proposed_method}
\vspace{-0.6em}
\subsection{Formalism of Motion Forecasting}
\vspace{-0.6em}
Consider a motion forecasting problem in a multi-agent environment with $M$ agents. Let's denote the states of the agents as $S_t=\{s_t^1, s_t^2, ... , s_t^M\}$ where $s_t^i=(x_t^i,y_t^i)$ are the 2D coordinates of agent $i$ at time $t$. The objective is to predict the states of agents  $T_{pred}$ time steps into the future from observations of the previous $T_{obs}$ time steps. The model takes as input $X^i=\{s_1, s_2, ... , s_{T_{obs}}\}$ and predicts $Y^i=\{s_{T_{obs}+1}, s_{T_{obs}+2}, ... , s_{T_{obs}+T_{pred}}\}$ for every agent $i$. 

\begin{figure}[ht]
\vskip 0.2in
\begin{center}
\centerline{\includegraphics[width=\textwidth]{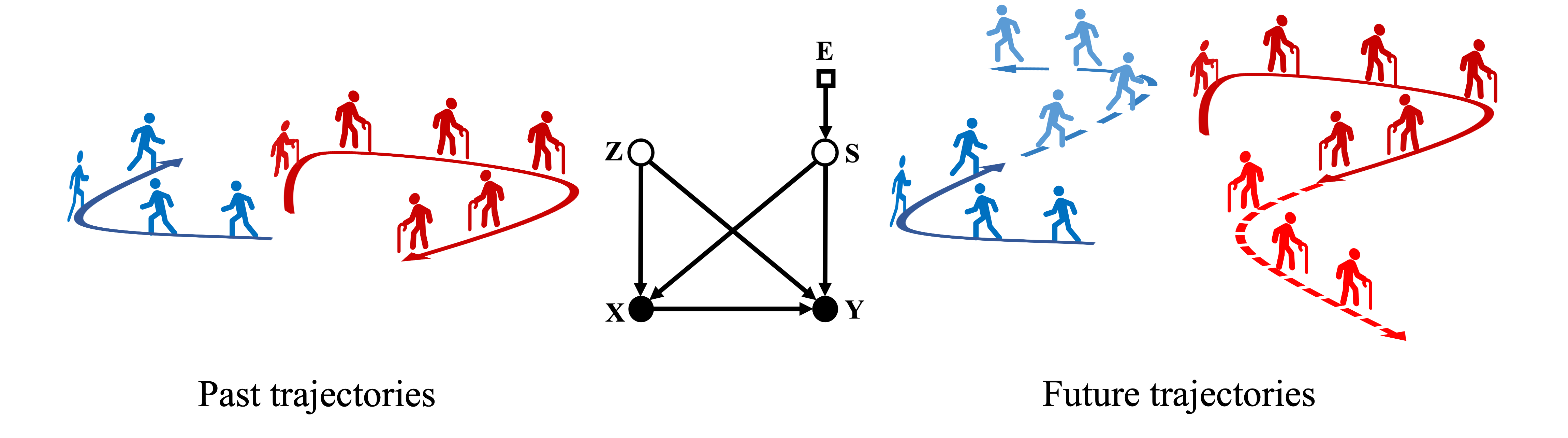}}
\caption{The proposed causal model (center). Filled circles are observed variables and empty shapes are the unobserved variables. $X$ and $Y$ represent past trajectories and future trajectories to be predicted, respectively. $Z$ represents invariant features common across domains, such as physical laws, while $S$ represents variant features specific to each environment, such as motion styles. Finally, $E$ is the selection variable, conditioning on $E$ allows us to switch between environments.}
\label{fig:causalmodel}
\end{center}
\vskip -0.2in
\end{figure}

Deep learning based motion forecasting models are often composed of three modules; an encoder, an interaction module and a decoder \cite{motionsurvey}. 
The encoder takes previously observed states $X^i$ and produces representations of the observed states. The interaction module receives the representations from the encoder to capture the social interactions (human-human), the physical interactions (human-space), or both. Eventually, the decoder takes as input, the interaction vectors and predicts the future states of the agent $Y^i$. 
It is worth mentioning that any type of neural network, which can handle sequence-to-sequence modeling such as Recurrent Neural Networks (RNNs) or temporal Convolutional Neural Networks (CNNs) could be employed as the Encoder and the decoder of the model.

Training data is usually collected from a set of $K$ environments or domains, $E=\{e_1, e_2, ..., e_k\}$. An example set of environments could be `Pedestrians crossing an intersection', `People standing in front of a department store', or `Pedestrians walking along the sidewalk'. In real-world datasets, it is common for the training and test environments to differ significantly \cite{counterfactualmotion}, such as in the widely-used ETH-UCY benchmark. Therefore, the i.i.d assumption does not always hold in practice and Empirical Risk Minimization (ERM) techniques cannot be used to train the neural networks. In the following section, we briefly present some concepts of causality applied in our proposed method. 

\vspace{-0.6em}
\subsection{Background in Causality}
\vspace{-0.6em}
Structural Causal Models (SCMs) are a way of describing causal features and their interactions, which are represented by Directed Acyclic Graphs (DAG) \cite{causaljudea}. We say that $X$ is a direct cause of $Y$ when there is a directed edge from $X$ to $Y$ in the DAG. The cause and effect relation $X \rightarrow Y$ tells us that changing the value of $X$ can result in a change in the value of $Y$, but that the reverse is not true. A causal model receives as inputs a set of qualitative causal assumptions (A), a set of queries concerning the causal relations among variables (Q), and experimental or non-experimental data (D), presumably consistent with (A). A causal model makes predictions about the behavior of a system. The outputs of a causal model are then a set of logical implications of (A), data-dependent claims (C) represented by the magnitude or likelihoods of the queries (Q), and a list of testable statistical implications (T) \cite{Transportability}.

A major drawback of deep learning models using statistical inference is that they do not consider the confounding bias. Normally, we would predict $p(Y|X)$, however, in the presence of confounding variables, $p(Y|X)$ cannot capture the true causal effect of $X$ on $Y$ and the predictions will be erroneous. Therefore, we are interested in the causal query $p(Y|do(X))$ which eliminates the confounding bias by surgically intervening on $X$. This causal query can be calculated using the Backdoor Criterion \cite{causaljudea}. Our set of causal assumptions are causal sufficiency and faithfulness \cite{markovblanket}.
Since we do not assume i.i.d samples, we use an augmented version of causal models called Selection Diagrams \cite{Transportability} to discriminate between samples of different domains. Selection variables $S$ are introduced to model the variations in each domain. Every member of $S$ corresponds to a mechanism by which the two populations differ. Therefore, we can condition on the values of $S$ for switching between domains \cite{Transportability}.
\vspace{-0.6em}
\subsection{Causal Formalism}
\vspace{-0.6em}
Our proposed SCM for motion forecasting is shown in Figure \ref{fig:causalmodel}. There are two causal variables in motion forecasting, that can affect the trajectories of the pedestrians: invariant features and variant features \cite{causalmotion}. \textit{Invariant features} do not vary across domains but can influence the trajectories of the pedestrians. These features can be associated with physical laws, traffic laws, social norms, and etc. 
In contrast, \textit{variant features} vary across domains and can be associated with the motion styles of the pedestrians in an environment \cite{causalmotion}. 
In our proposed causal model we consider four endogenous variables: $S$ to represent variant features, $Z$ for invariant features, $X$ for past trajectories, and $Y$ for future trajectories.

We also introduce an additional exogenous variable $E$ as the selection variable \cite{Transportability} to account for the changing factors in each environment. The selection variable acts as an identifier of an environment such that the conditional probability of $p(X,Y|E=e_1)$ represents the distribution of the samples in an environment having the selection variable of $e_1$. In other words, we assume that all members of the dataset are sampled from a parent distribution over $X$, $Y$, and $E$. Furthermore, we assume that the proposed model is \textit{causally sufficient}. That is, it explains all the dependencies without adding further causal variables. To build a causal graph, we reason about causal edges connecting the causal variables:\\
\textbf{(1)} There should be an edge from $S$ to $X$ and $Y$ because motion styles can influence the speed of the pedestrians.\\
\textbf{(2)} There should be an edge from $Z$ to $X$ and $Y$ because social norms can influence how closely pedestrians move next to each other.\\
\textbf{(3)} There should be an edge from $X$ to $Y$ because the location in the past determines where the pedestrian is going to be in the future. \\
\textbf{(4)} $S$ varies in each domain, hence, there should be an edge from selection variable $E$ to $S$  to account for all the changing factors in each domain.\\ 

\vspace{-0.9em}
\subsection{Learning Latent Variables with Variational Inference}
\vspace{-0.6em}
According to Figure \ref{fig:causalmodel}, $S$ and $Z$ confound the causal effect of $X$ on $Y$ and the backdoor paths are $Y \leftarrow S \rightarrow X$ and $Y \leftarrow Z \rightarrow X$. Therefore, we need to use the backdoor criterion to eliminate the confounding effect. To this end, the causal effect of $X$ on $Y$ (i.e., $p(Y|do(X))$) is calculated. Before calculating this causal query, we need to recall the \textit{S-admissibility} criterion:
\begin{definition}
\label{sadmissible}
``A set $T$ of variables satisfying $(Y \independent S|T, X)$ in $D_{\Bar{X}}$ will be called \textit{S-admissible} (with respect to the causal effect of $X$ on $Y$).''\cite{Transportability}
\end{definition}
Where $D_{\Bar{X}}$ denotes an augmented SCM in which we intervene on X (i.e., the arrows entering X are deleted). According to this definition, in Figure \ref{fig:causalmodel}, the set ${S}$ will be an E-admissible set, which means once we know about $S$, the predictions ($Y$) will not be dependent on the environment ($E$) we're in. 

We can now proceed to calculate $p(Y|do(X))$:
\begin{equation}
\label{eq:backdoor}
\begin{split}
      &p(Y|do(X), E=e_i)  \\
      &=\int p(Y|X,S,Z,E)p(S|E)p(Z|E) ds dz \\
      &=\int p(Y|X,S,Z)p(S|E)p(Z) ds dz  \\
      &= E_{p(S|E), p(Z)}[p(Y|X,S,Z)],
\end{split}
\end{equation}  
where $p(Y|do(X), E=e_i)$ is the causal effect of $X$ on $Y$ in environment $i$. The first line follows from the backdoor criterion and the second line follows from the fact that $S$ is E-admissible \cite{Transportability} and $X$ is a collider. Equation \ref{eq:backdoor} implies that in order to calculate the causal effect of $X$ on $Y$ in every environment, we need to stratify over $S$ and $Z$. Although we don't know the true posterior of the latent variables, we will approximate them using mean-field variational inference \cite{Goodfellowdeeplearning}. 

The standard log-likelihood to train the model is given by:
\begin{equation}
\label{eq:mll}
\begin{split}
\underset{p}{max} \:  E_{p^*(x,y)}[\log p(x,y)]
\end{split}
\end{equation}  
where $p^*(x,y)$ is the distribution of samples in the dataset. Calculating $p(x,y)$ is intractable since $p(x,y)=\sum_{e \in E}\int(p(x,y,s,z,e)dsdz$ where $E$ is the set of environments. Mean-field variational inference is used to approximate the true posterior distribution of latent variables. 

For this purpose, the \textit{Evidence Lower Bound (ELBO)} function is used to train the model:
\begin{equation}
\label{eq4}
   \underset{p, q}{max} \: E_{p^*(x,y)}\left[E_{q(s,z, e|x,y)}\left[\log \frac{p(x,y,s,z,e)}{q(s,z,e|x,y)}\right]\right]
\end{equation}
Theoretically, the ELBO function will drive $q(s,z,e|x,y)$ towards its target $p(s,z,e|x,y)$ and the objective function in Equation \ref{eq4} will become Equation \ref{eq:mll}. The approximate model is still intractable because we don't have access to the future trajectories in the test domain. Therefore, we replace it with $q(s,z,e,y|x)$ and the loss function becomes:
\begin{equation}
\label{eq:twoterms}
\begin{split}
   &\underset{p, q}{max} \: E_{p^*(x)}E_{p^*(y|x)}[\log q(y|x)] + \\
   &E_{p^*(x)}\Biggl[E_{q(s,z,e,y|x)}\left[\frac{p^*(y|x)}{q(y|x)}\log \frac{p(x,y,s,z,e)}{q(s,z,y,e|x)}\right]\Biggr]
\end{split}
\end{equation}
Assuming that $q(y|x)$ is Gaussian, the first term in the loss function of Equation \ref{eq:twoterms} would be the negative of Mean Squared Error (MSE). Eventually, this term will drive $q(y|x)$
towards $p^*(y|x)$ and the second term will become a lower bound of $\log p(x)$ as stated in Theorem \ref{thm:lowerbound}, which we prove below.
\begin{theorem}
\label{thm:lowerbound} 
Let $p(x,y,s,z,e)$ be the joint distribution of latent and observed variables and $q(s,z,y,e|x)$ be the approximate posterior of latent variables and future trajectories given the past trajectories in GCRL, a lower bound on the log-likelihood of past trajectories is:
\begin{equation*}
    E_{q(s,z,e,y|x)}\left[\log \frac{p(x,y,s,z,e)}{q(s,z,y,e|x)}\right] \leq \log p(x)
\end{equation*}
\end{theorem}
\begin{proof} 
According to the Causal Markov condition \cite{causal_survey}, we can factorize $q(s,z,e,y|x)=q(y|x,s,z)q(z|x)q(s|x,e)q(e|x)$ and $p(x,s,z,y,e)=p(y|x,s,z)p(x|s,z)p(s|e)p(z)p(e)$. The first approximate model, $q(y|x,s,z)$, can be replaced with $p(y|x,s,z)$ since it is known. Secondly, since $q(s|x,e)$ is an approximation of its true posterior, we assume it to be $q(s|x)$.  Therefore:
\begin{equation}
\label{eq:elbof}
\begin{split}
   &E_{q(s,z,y,e|x)}\left[\log \frac{p(x,y,s,z,e)}{q(s,z,y,e|x)}\right] = \\
   &E_{q(s,z,e|x)}\left[\log \frac{p(x|s,z)p(s|e)p(z)p(e)}{q(z|x)q(s|x)q(e|x)}\right] = \\
   &E_{q(s|x), q(z|x)}\left[\log p(x|s,z)\right] - \\ 
   &KL(q(z|x)||p(z)) - KL(q(s|x)||p(s)) - \\
   &KL(q(e|x)||p(e|s)) \leq \log p(x)
\end{split}
\end{equation}
The detailed proof is given in Appendix \ref{appA}.
\end{proof}
As shown in Equation \ref{eq:elbof}, $q(e|x)$ will be driven towards $p(e|s)$, hence, it is sensible to model $q(e|x)$ with $p(e|s)$. By leveraging the Causal Markov condition and replacing $q(e|x)$ with $p(e|s)$, we can obtain a lighter inference model (q) and the loss function in Equation \ref{eq:twoterms} becomes:
\begin{equation}
\label{eq:finalloss}
\begin{split}
     \underset{p, q}{max} &~ E_{p^*(x,y)}\Biggl[\log q(y|x) + \frac{1}{q(y|x)}\\
     &\quad E_{q(s|x),q(z|x)}\Bigl[p(y|x,s,z)
     \log \frac{p(x|s,z)p(s)p(z)}{q(s|x)q(z|x)}\Bigr]\Biggr]  ,
\end{split}
\end{equation}
where $p(s) = \sum_{e \in E}p(s|e)p(e)$ which means that $S$ has a Gaussian mixture prior and $q(y|x)=E_{q(s|x), q(z|x)}[p(y|x,s,z)]=p(y|do(x))$ which can be calculated by ancestral sampling.  The expectations in Equation \ref{eq:finalloss} can be estimated using the Monte-Carlo method after applying the re-parametrization trick \cite{repar}. Consequently, GCRL learns: \textbf{(1)} To minimize the distance between groundtruth future trajectories and predicted future trajectories via maximizing $\log q(y|x)$, \textbf{(2)} To eliminate the confounding effect by estimating the causal effect of $X$ on $Y$ via $p(y|do(x))$, \textbf{(3)} To reconstruct past trajectories via maximizing $\log p(x|s,z)$, \textbf{(4)} Invariant representations via maximizing $\log \frac{p(z)}{q(z|x)}$, \textbf{(5)} Variant representations via maximizing $\log \frac{p(s)}{q(s|x)}$. Furthermore, since GCRL learns to predict the future trajectories with a generative approach, it can tackle the multi-modality of trajectories. 
\begin{figure*}[ht]
\vskip 0.2in
\begin{center}
\centerline{\includegraphics[width=\textwidth]{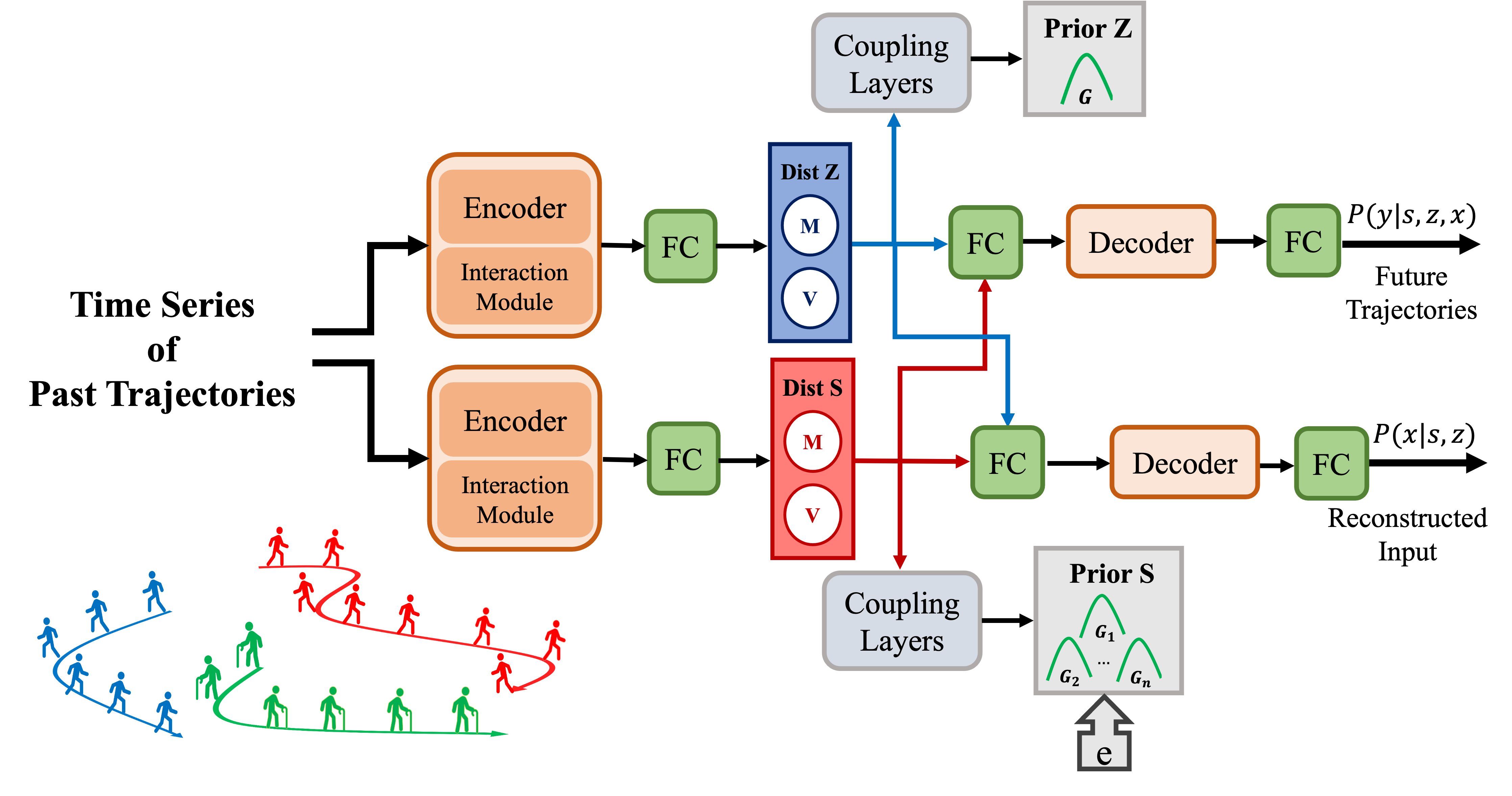}}
\caption{A general overview of the proposed method. The approximate posteriors of the latents are estimated using the encoded past trajectories and the priors of the latents are calculated using the coupling layers. The encoder, interaction module, and decoder of the model can be replaced by motion forecasting models that use an EID architecture.}
\label{fig:model}
\end{center}
\vskip -0.2in
\end{figure*}

Gaussian Mixture priors have been previously used in popular VAE models \cite{VaDE, SVAE}. Variational models are susceptible to poor local minima and posterior collapse, it is important to show the identifiability of latent variables \cite{vae-identifiability}. Variational models with GMM priors are proven to be identifiable \cite{vae-identifiability}, hence, using a Gaussian Mixture prior for $S$ aligns with the theory as well. Furthermore, Gaussian Mixture Models (GMMs) are universal approximators, hence, $q(s|x)$ will be capable of producing arbitrary variant features. To obtain a better likelihood for latent variables, we use coupling layers \cite{rnvp} to learn rich priors $p(s|e_i)$ and $p(z)$ from simple priors such as standard Normal distributions. These proirs are referred to as \textit{flow priors} and are also used in VAEs \cite{learning_prior}. 

A general overview of our model is shown in Figure \ref{fig:model}. The encoded past trajectories are used to model the approximate posteriors of the latents. We sample from these posteriors to reconstruct the past trajectories and predict the future trajectories. As shown in Figure \ref{fig:model}, any motion forecasting model with an EID architecture can be used with our method provided that the decoder is a piece-wise affine injective function as this is required to obtain the weakest form of identifiability (i.e., identifiability up to an affine transformation \cite{vae-identifiability}).
\vspace{-0.6em}
\subsection{Domain Adaptation Method}
\vspace{-0.6em}
After learning the causal mechanisms and causal variables using Equation \ref{eq:finalloss}, we know that $q(z|x)$ will generate representations with a single Gaussian distribution and $q(s|x)$ will generate representations with a Gaussian Mixture Model (GMM) distribution. Therefore, as illustrated in Figure \ref{fig:dist}, all representations generated by $q(z|x)$ will be in the same range, whereas the representations of $q(s|x)$ will form clusters, each modeled by a component of the GMM. Since $Z$ is invariant, we can directly transfer it to the new domain without any fine-tuning. However, $S$ can be interpreted as a weighted sum of the representations learnt from different environments of the training domains, which may be used in the test domains as well. Depending on how related the test domains are to the training domains, we may need to fine-tune the components of the GMM and obtain a new prior for $S$.  
Thus, to fine-tune the model at inference time, we reuse the loss function in Equation \ref{eq:finalloss} without the regularizing $Z$ posterior by omitting $q(z|x)$. Eventually, $q(s|x)$ will be driven towards the new prior and compensate for the domain shift in the test domain. The models to predict future trajectories $p(y|x,s,z)$ and to reconstruct past trajectories $p(x|s,z)$ also needs to be fine-tuned as the samples of $q(s|x)$ will be updated. 

Consequently, we only fine-tune the models of $q(s|x)$, $p(s)$, $p(y|x,s,z)$, and $p(x|s,z)$ while the models for, $p(z)$, and $q(z|x)$ can be arbitrarily complex as it is not required to fine-tune them in the test domain, but all the other models should be as simple as possible.

\begin{figure}[ht]
\vskip 0.2in
\begin{center}
\centerline{\includegraphics[width=0.7\columnwidth]{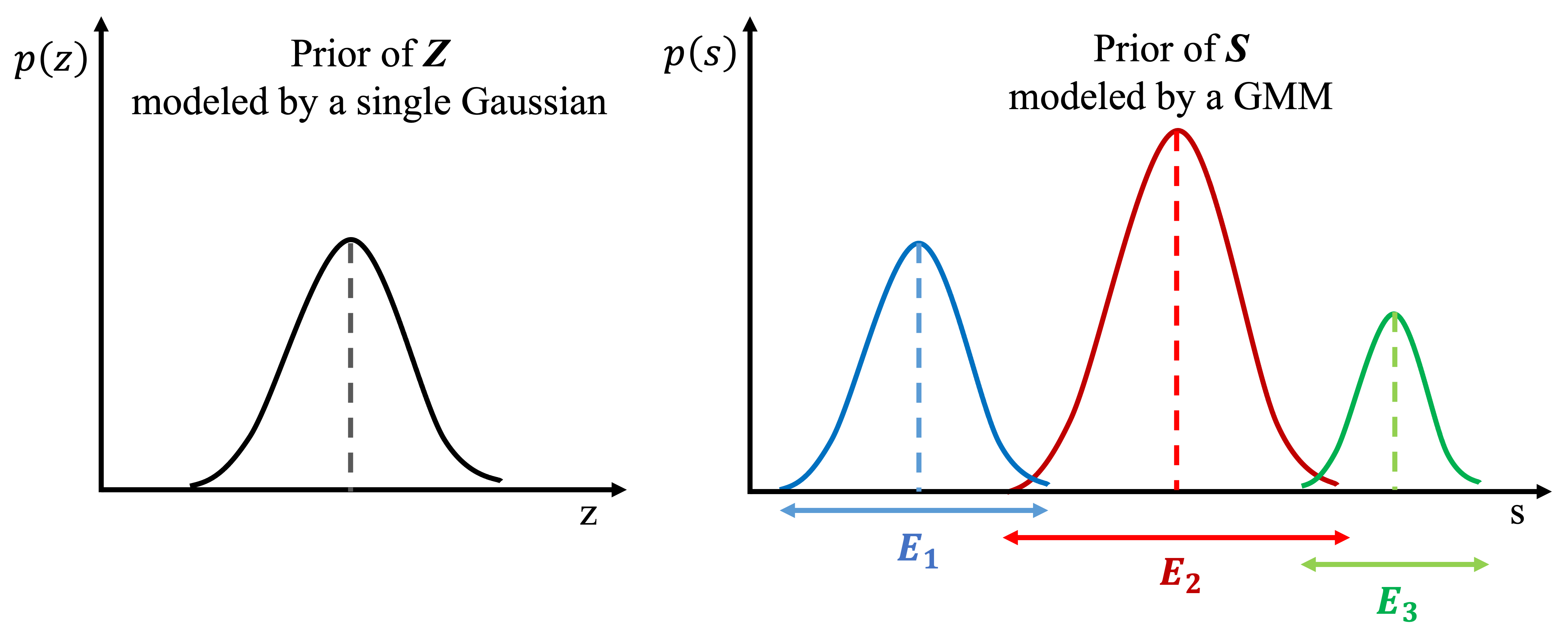}}
\caption{The priors of variant and invariant features. $E_i$ represents environment $i$.}
\label{fig:dist}
\end{center}
\vskip -0.2in
\end{figure}

\vspace{-0.6em}
\section{Experiments}
\vspace{-0.6em}
In our experiments we evaluate our models using two metrics often used in evaluation of the motion forecasting models: 
\begin{equation} 
\label{eq:ade}
	ADE = \frac{\sum_{i=1}^{n}\sum_{t=T_{obs}+1}^{T_{pred}}((\hat{x}_i^t-x_i^t)^2+(\hat{y}_i^t-y_i^t)^2)^ \frac{1}{2}}{n(T_{pred}-(T_{obs}+1))} 
\end{equation}
\begin{equation} 
\label{eq:fde}
	FDE = \frac{\sum_{i=1}^{n}((\hat{x}_i^{T_{pred}}-x_i^{T_{pred}})^2+(\hat{y}_i^{T_{pred}}-y_i^{T_{pred}})^2)^\frac{1}{2}}{n}, 
\end{equation}
where $\hat{x}_i^t$ and $\hat{y}_i^t$ are the \textit{predicted} horizontal and vertical coordinates of the pedestrian at time step $t$, while $x_i^t$ and $y_i^t$ are the \textit{actual} horizontal and vertical coordinates of the pedestrian at time step $t$.
\vspace{-0.6em}
\subsection{Datasets}
\vspace{-0.4em}
\paragraph{ETH-UCY dataset}This dataset contains the trajectory of 1536 detected pedestrians captured in five different environments \textit{\{hotel, eth, univ, zara1, zara2\}}. All trajectories in the dataset are sampled every 0.4 seconds. Following the experimental settings of \cite{causalmotion, counterfactualmotion, STGAT}, we also use a leave-one-out approach for training and evaluating our model so to predict the future 4.8 seconds (12 frames), we utilize the previously observed 3.2 seconds (8 frames).
\vspace{-0.6em}
\paragraph{Synthetic dataset}This dataset published in \cite{causalmotion} contains the trajectories of pedestrians in circle-crossing scenarios \cite{circle_cross} where the minimum separation distance of pedestrians differ in each domain.  There are 8 domains in the dataset with separation distances  $\{0.1, 0.2, 0.3, 0.4, 0.5, 0.6, 0.7, 0.8\}$ meters.  Each domain contains 10,000 trajectories for training, 3,000 trajectories for validation, and 5,000 trajectories for testing. 
\vspace{-0.6em}
\subsection{Robustness}
\vspace{-0.6em}
To evaluate the robustness of our model in the presence of spurious correlations, we compare our method with \cite{causalmotion, counterfactualmotion}. For a fair comparison, we use the STGAT \cite{STGAT} as our baseline model. Although, ETH-UCY contains five environments, it is not trivial to pinpoint the shifts in each environment. Therefore, we add a third dimension to the coordinates of the pedestrians, which measures observation noise and is modeled as in \cite{causalmotion}:
\begin{equation} 
\label{eq:noise}
\begin{split}
    &\gamma_t := (\Dot{x}_{t+\delta t} - \Dot{x}_t)^2 + (\Dot{y}_{t+\delta t} - \Dot{y}_t)^2 \\
    &\sigma_t := \alpha (\gamma_t + 1),
\end{split}
\end{equation}
where $\Dot{x}_t = x_{t+1} - x_{t}$ and $\Dot{y}_t = y_{t+1} - y_{t}$ reflect the velocity of the pedestrians within the temporal window length of $\delta t = 8$ and $\alpha$ is the noise intensity. For the training domains $\alpha \in \{1, 2, 4, 8\}$ while for the test domain $\alpha \in \{8, 16, 32,64\}$. The test domain is the \textit{eth} environment for this experiment. Since the value of $\alpha$ for the third input dimension in training domains were $\{1, 2, 4, 8 \}$, the samples of the test domain with $\alpha \in \{16, 32, 64\}$ can be considered as out-of-distribution samples. To evaluate other methods in presence of observation noise, we have used the publicly available code from \cite{causalmotion}. The results in Table \ref{tab:table1} demonstrate the robustness of our method against observation noise while performing comparably with other motion forecasting models for low $\alpha$. Since our proposed method also learns to reconstruct inputs, it eliminates the effect of noise by reconstructing uncorrupted inputs, hence, it is not sensitive to noise. 

\begin{table*}[h]
\caption{Robustness of different methods in the ETH-UCY dataset with controlled spurious correlation.}
\def\arraystretch{1.2}
\label{tab:table1}
\vskip 0.15in
\begin{center}
\begin{small}
\begin{sc}
\begin{tabular}{lcccc}
\textbf{Method} & \textbf{ADE/FDE} & \textbf{ADE/FDE} & \textbf{ADE/FDE} & \textbf{ADE/FDE} \\
\toprule
 &  $\alpha=8$ & $\alpha=16$ & $\alpha=32$ & $\alpha=64$ \\ 
\midrule
Baseline \cite{STGAT} & \textbf{0.80/1.37} & 2.15/3.80 & 2.64/4.44 & 2.68/4.48 \\
Counterfactual \cite{counterfactualmotion} & 0.80/1.59 & 1.62/2.68 & 2.32/3.90 & 2.68/4.52 \\
Invariant $\lambda=1.0$ \cite{causalmotion} & 0.94/1.65 & 1.04/\textbf{1.76} & 1.52/2.55 & 1.96/3.35 \\
Invariant $\lambda=3.0$ \cite{causalmotion} & 0.91/1.67 & 0.99/1.87 & 1.18/2.20 & 1.27/2.33 \\
Invariant $\lambda=5.0$ \cite{causalmotion} & 0.98/1.79 & 1.00/1.83 & 1.06/1.90 & 1.56/2.58 \\
GCRL (ours)    & 0.97/1.8 & \textbf{0.97}/1.8 & \textbf{0.97/1.8} &  \textbf{0.97/1.8}     \\
\bottomrule
\end{tabular}
\end{sc}
\end{small}
\end{center}
\vskip -0.1in
\end{table*}

\vspace{-0.6em}
\subsection{Domain Generalization}
\vspace{-0.6em}
In this experiment, we evaluate the generalizability of our proposed method using the synthetic dataset \cite{causalmotion}. We will refer to \cite{causalmotion} as IM in the remainder of the paper. For a fair comparison with IM, we use a PECNet \cite{PECNet} variant as our base model, train the model with five seeds and report the mean and standard deviation of the results. We will use the PECNet variant as our base model in the subsequent experiments. The Minimum Separation Distances (MSD) in the training and test domains are $\{0.1, 0.3, 0.5\}$ and $\{0.1, 0.2, 0.3, 0.4, 0.5, 0.6, 0.7, 0.8\}$ meters, respectively. 

\begin{figure}[ht]
\vskip 0.2in
\begin{center}
\centerline{\includegraphics[width=0.5\textwidth]{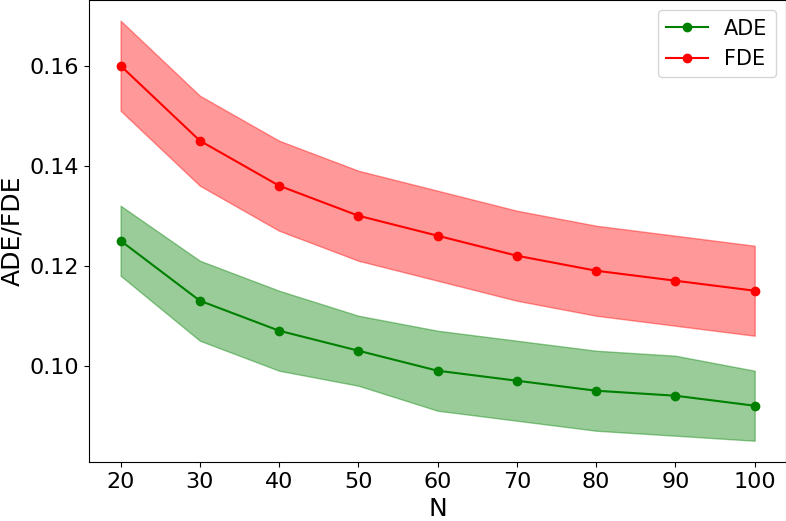}}
\caption{Comparison of ADE/FDE for different values of $N$.}
\label{fig:best_k}
\end{center}
\vskip -0.2in
\end{figure}

\begin{figure}[ht]
\vskip 0.2in
\begin{center}
\centerline{\includegraphics[width=0.5\textwidth]{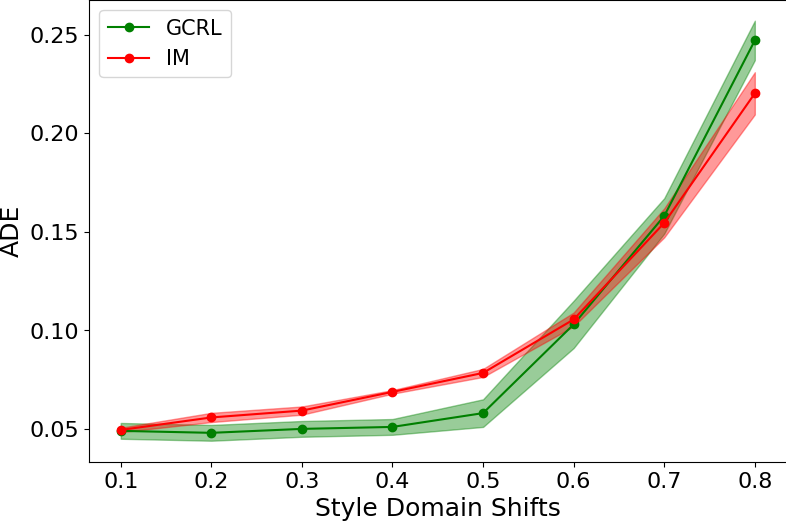}}
\caption{Domain generalization with different style domain shifts. The mean and standard deviation of ADEs are calculated and plotted for 5 seeds.}
\label{fig:generalize}
\end{center}
\vskip -0.2in
\end{figure}
Since GCRL is a generative approach, we can generate multiple future trajectories per sample and select the best of them to tackle the multi-modality of trajectories as practiced in \cite{SocialGAN, STGAT, counterfactualmotion}. Therefore, we use a hyper-parameter N in testing to determine the number of generated trajectories per sample. Figure \ref{fig:best_k} illustrates the significant impact that a generative approach can have in the performance. We will use $N=100$ in the following experiments.

As illustrated in Figure \ref{fig:generalize}, our method is more robust to domain shifts compared to IM, and it is achieving slightly better ADE, which is 8.8\% on average. 
It is evident that for OOD-Inter cases where the test domain shift is within the range of training domain shifts e.g., test domain shift=0.4, GCRL is reusable since ADE is insensitive to the domain shifts. On the other hand, for the test domain shifts out of the range of training domain shifts, the OOD-Extra cases, the model needs to be fine-tuned.
\vspace{-0.6em}
\subsection{Domain Adaptation}
\vspace{-0.6em}
In this experiment, we evaluate the efficiency of our proposed method in knowledge transfer using the synthetic dataset for an OOD-Extra case. We train IM and GCRL with the previous setting and fine-tune different components of the model using a small number of batches from the test domain. The batch size is 64, hence, the number of samples used in fine-tuning will be $\{1, 2, 3, 4, 5, 6\} \times 64$. For IM, we fine-tune it using the best setting reported in the paper. For GCRL, we fine-tune our models for $p(y|x,s,z)$, $p(x|s,z)$, $p(s)$ and $q(s|x)$. 

\begin{figure}[ht]
\vskip 0.2in
\begin{center}
\centerline{\includegraphics[width=0.5\textwidth]{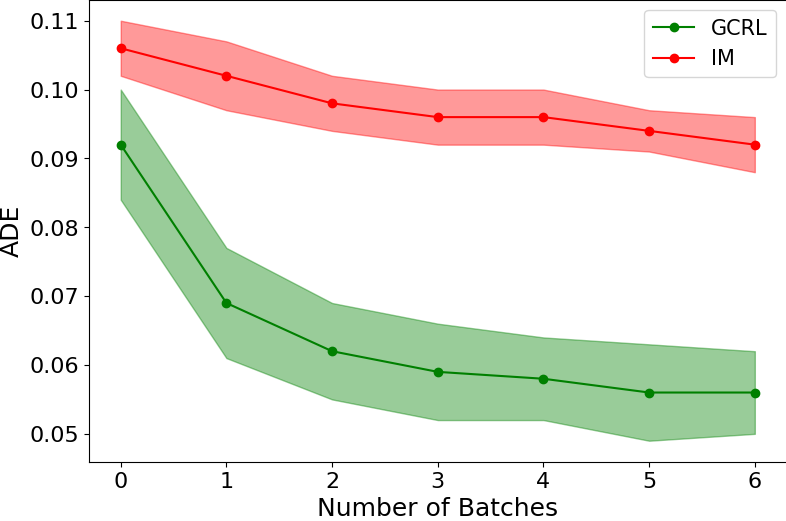}}
\caption{Domain Adaptation with different number of batches}
\label{fig:adaptation}
\end{center}
\vskip -0.2in
\end{figure}

As shown in Figure \ref{fig:adaptation}, GCRL adapts to the new environment faster than IM and it is more robust to OOD-Extra shifts. In our experiment we fine-tune only the weights of GMM prior, which improves the ADE by 34.3\% from IM on average.

\vspace{-0.6em}
\subsection{Identifiability}
\vspace{-0.6em}
To evaluate the identifiability of $S$ and $Z$, we train GCRL with five different seeds on the synthetic dataset. With a GMM prior and a piecewise-affine injective decoder, identifiability up to an affine transformation can be achieved. Therefore, we calculate MCC between pairs of learnt latents before and after applying an affine transformation $f$, which we refer to as Strong MCC and Weak MCC, respectively. $f$ can be learnt via least-squares minimization. 

The results in Table \ref{tab:table2} support the identifiability of $S$ as expected. $Z$, however, is not identifiable from observed trajectories and more data from the environments is required. This finding also aligns with the capability of past trajectories in explaining the features that $Z$ is supposed to represent such as physical laws. For example, in order to learn that pedestrians avoid obstacles in their paths, one needs to provide images of the environments. Despite the poor identifiability of $Z$, the empirical results in the next section indicate the invariance of $Z$ across domains. 

\begin{table}[ht]
\caption{MCC of $S$ and $Z$ after and before applying an affine transformation.}
\def\arraystretch{1.1}
\label{tab:table2}
\vskip 0.15in
\begin{center}
\begin{small}
\begin{sc}
\begin{tabular}{cccc}
\textbf{Weak MCC of S}      & \textbf{Weak MCC of Z}       & \textbf{Strong  MCC of S}    & \textbf{Strong  MCC of Z} \\
\toprule
0.956 & 0.049 & -0.16 &  -0.025 \\
\bottomrule
\end{tabular}
\end{sc}
\end{small}
\end{center}
\vskip -0.1in
\end{table}
\vspace{-0.6em}
\subsection{Ablation studies}
\vspace{-0.6em}
In this experiment, we examine the individual contribution of the components of GCRL. First, we remove the coupling layers from the priors and replace them with learnable parameters. Second, we reconstruct past trajectories and predict future trajectories using only samples from $q(z|x)$. Third, we use only the samples of $q(s|x)$. We train the models on the synthetic dataset with $\{0.1, 0.3, 0.5\}$ shifts and test it on the domain with 0.6 shift. 

\begin{table}[ht]
\caption{The effect of different components on performance.}
\def\arraystretch{1.2}
\label{tab:table3}
\vskip 0.15in
\begin{center}
\begin{small}
\begin{sc}
\begin{tabular}{lcc}
 \textbf{Models} & \textbf{ADE} & \textbf{FDE} \\
 \toprule
Vanila GCRL & 0.0871 & 0.1047  \\

No coupling layers & 0.0772 & 0.0916  \\

Z only & 0.1054 & 0.1347  \\

S only & 0.2188 & 0.2418  \\
\bottomrule
\end{tabular}
\end{sc}
\end{small}
\end{center}
\vskip -0.1in
\end{table}

As shown in Table \ref{tab:table3}, the model with only $Z$ performs on par with the default model, however, the performance deteriorates when using only $S$. It can be concluded that $Z$ is invariant across domains as expected, however, it lacks information, which leads to worse ADE. The model without coupling layers performs better than the default model for synthetic dataset, which indicates that fewer parameters would suffice for certain environments.

\vspace{-1.0em}
\section{Conclusions}
\vspace{-0.6em}
We propose a method that leverages causality to learn meaningful features that can increase the robustness and transferability of deep learning models. In presence of spurious correlation, we demonstrated the robustness of our method while other human trajectory prediction models performed poorly compared to our method. Furthermore, our augmented causal model was able to enhance the transferability in a zero-shot and low-shot settings. It can be concluded from our results that incorporating causality in deep learning is a promising research direction towards robustness and explainability of deep learning models. 

\bibliographystyle{unsrtnat}
\bibliography{references}  

\newpage
\appendix
\onecolumn
\section{Methodology details}
\label{appA}
\paragraph{Lower bound on the log likelihood of past trajectories} The objective function in GCRL is given in Equation \ref{eq4}. Since the future trajectories are unknown at test time, the approximate posterior $q(s,z,e|x,y)$ is intractable. Therefore, we rewrite the Equation \ref{eq4} such that the approximate posterior is conditioned on only x:
\begin{equation}
\label{eq:proof_twoterms}
\begin{split}
   &\underset{p, q}{max} \: E_{p^*(x,y)}\left[E_{q(s,z, e|x,y)}\left[\log \frac{p(x,y,s,z,e)}{q(s,z,e|x,y)}\right]\right]= \\
   &\underset{p, q}{max} \: E_{p^*(x,y)}\left[E_{q(s,z, e|x,y)}\left[\log \frac{p(x,y,s,z,e)q(y|x)}{q(s,z,e,y|x)}\right]\right]= \\
   & \underset{p, q}{max} \: E_{p^*(x,y)}\left[\int \left[\log \frac{p(x,y,s,z,e)q(y|x)}{q(s,z,e,y|x)}\right]q(s,z, e|x,y)dsdzde\right] = \\
   & \underset{p, q}{max} \: E_{p^*(x,y)}\left[\int \left[\log \frac{p(x,y,s,z,e)}{q(s,z,e,y|x)}\right]q(s,z, e|x,y)dsdzde + \int \left[ \log q(y|x) \right] q(s,z, e|x,y)dsdzde \right] = \\
   & \underset{p, q}{max} \: E_{p^*(x,y)}\left[\int \left[\log \frac{p(x,y,s,z,e)}{q(s,z,e,y|x)}\right]q(s,z, e|x,y)dsdzde + \log q(y|x)  \right] = \\
   & \underset{p, q}{max} \: E_{p^*(x,y)}\left[\int \left[\log \frac{p(x,y,s,z,e)}{q(s,z,e,y|x)}\right]\frac{q(s,z, e,y|x)}{q(y|x)}dsdzde + \log q(y|x)  \right] = \\
   & \underset{p, q}{max} \: E_{p^*(x,y)}\left[\int \left[\frac{1}{q(y|x)}\log \frac{p(x,y,s,z,e)}{q(s,z,e,y|x)}\right]q(s,z, e,y|x)dsdzde + \log q(y|x)  \right] = \\
   & \underset{p, q}{max} \: \int \left[\frac{1}{q(y|x)}\log \frac{p(x,y,s,z,e)}{q(s,z,e,y|x)}\right]p^*(x,y)q(s,z, e,y|x)dsdzdedxdy + \int \left[\log q(y|x) \right] p^*(x,y)dxdy  = \\
   & \underset{p, q}{max} \: \int \left[\frac{1}{q(y|x)}\log \frac{p(x,y,s,z,e)}{q(s,z,e,y|x)}\right]p^*(y|x)p^*(x)q(s,z, e,y|x)dsdzdedxdy + \int \left[\log q(y|x) \right] p^*(y|x)p^*(x)dxdy  = \\
   & \underset{p, q}{max} \: \int \left[\frac{p^*(y|x)}{q(y|x)}\log \frac{p(x,y,s,z,e)}{q(s,z,e,y|x)}\right]q(s,z, e,y|x)p^*(x)dsdzdedxdy + \int \left[\log q(y|x) \right] p^*(y|x)p^*(x)dxdy  = \\
   & \underset{p, q}{max} \: \int E_{q(s,z, e,y|x)}\left[\frac{p^*(y|x)}{q(y|x)}\log \frac{p(x,y,s,z,e)}{q(s,z,e,y|x)}\right]p^*(x)dx + \int E_{p^*(y|x)}\left[\log q(y|x) \right] p^*(x)dx = \\
   &\underset{p, q}{max} \: E_{p^*(x)}E_{p^*(y|x)}[\log q(y|x)] + 
   E_{p^*(x)}[E_{q(s,z,e,y|x)}[\frac{p^*(y|x)}{q(y|x)}\log \frac{p(x,y,s,z,e)}{q(s,z,y,e|x)}]]
\end{split}
\end{equation}

which gives us the Equation \ref{eq:twoterms}. According to Equation \ref{eq:proof_twoterms}, $q(y|x) $ will be derived towards $p^*(y|x)$ as desired and the objective function will become:

\begin{equation}
\label{eq:proof_lowerbound}
\begin{split}
   &\underset{p, q}{max} \: E_{p^*(x)}[E_{q(s,z,e,y|x)}[\frac{p^*(y|x)}{p^*(y|x)}\log \frac{p(x,y,s,z,e)}{q(s,z,y,e|x)}]]= \\
   &\underset{p, q}{max} \: E_{p^*(x)}[E_{q(s,z,e,y|x)}[\log \frac{p(x,y,s,z,e)}{q(s,z,y,e|x)}]]= \\
   &\underset{p, q}{max} \: E_{p^*(x)}[\int [\log \frac{p(x,y,s,z,e)} {q(s,z,y,e|x)}]q(s,z,e,y|x)dsdzdedy]= \\
   &\underset{p, q}{max} \: E_{p^*(x)}[\int [\log \frac{p(y|s,z,x)p(x|s,z)p(z)p(s|e)p(e)}{q(y|s,z,x)q(z|x)q(s|x)q(e|x)}]q(y|s,z,x)q(z|x)q(s|x)q(e|x)dsdzdedy]= \\
   &\underset{p, q}{max} \: E_{p^*(x)}[\int [\log \frac{p(x|s,z)p(z)p(s|e)p(e)}{q(z|x)q(s|x)q(e|x)}]p(y|s,z,x)q(z|x)q(s|x)q(e|x)dsdzdedy]= \\
   &\underset{p, q}{max} \: E_{p^*(x)}[\int [\log \frac{p(x|s,z)p(z)p(s|e)p(e)}{q(z|x)q(s|x)q(e|x)}]q(z|x)q(s|x)q(e|x)dsdzde]= \\
   &\underset{p, q}{max} \: E_{p^*(x)}[\int [\log \frac{p(x|s,z)p(z)p(s)p(e|s)}{q(z|x)q(s|x)q(e|x)}]q(z|x)q(s|x)q(e|x)dsdzde]= \\
   &\underset{p, q}{max} \: E_{p^*(x)}[E_{q(s,z|x)} [\log p(x|s,z)] + E_{q(z|x)} [\log \frac{p(z)}{q(z|x)}] + E_{q(s|x)} [\log \frac{p(s)}{q(s|x)}] + E_{q(e|x)} [\log \frac{p(e|s)}{q(e|x)}]] \leq \\
   &\underset{p, q}{max} \: E_{p^*(x)}[E_{q(s,z|x)} [\log p(x|s,z)]] \leq \\
   &\underset{p, q}{max} \: E_{p^*(x)}[\log E_{q(s,z|x)} [p(x|s,z)]] = E_{p^*(x)}[\log p(x)]
\end{split}
\end{equation}

which is evidently the lower bound of $\log p(x)$. $KL$ denotes the KL-divergence. Since the natural logarithm is a concave function, the second inequality follows from Jensen's inequality. The third equality in Equation \ref{eq:proof_lowerbound} follows from the Causal Markov condition that $q(s,z,e,y|x)=q(y|x,s,z)q(z|x)q(s|x,e)q(e|x)$ and $p(x,s,z,y,e)=p(y|x,s,z)p(x|s,z)p(s|e)p(z)p(e)$. Furthermore, we are using mean-field variational inference in which latent variables are assumed to be independent i.e., $q(s|e,x)=q(s|x)$. The last equality holds because eventually $q(s|x)=p(s)$ and $q(z|x)=p(z)$.
\paragraph{GCRL objective function} According to Equation \ref{eq:proof_lowerbound}, eventually, $q(e|x)$ will be derived towards $p(e|s)$. Therefore, it is sensible to assume that $q(e|x)=p(e|s)$. This is useful because it helps us to obtain a GMM prior and avoid modelling of the approximate posterior of E. This can be shown via Bayes rule and replacing :

\begin{equation}
\label{eq:bayes}
    q(e|x)=p(e|s)=\frac{p(s|e)p(e)}{p(s)}=\frac{p(s|e)p(e)}{\sum_e p(s|e)p(e)}
\end{equation}

Since we assume $p(e)$ and $p(s|e)$ are known, we can avoid modelling $q(e|x)$. By replacing $q(e|x)$ in Equation \ref{eq:twoterms} and using the Causal Markov condition, we can obtain a GMM prior for S:

\begin{equation}
\label{eq:proof_twoterms}
\begin{split}
   &\underset{p, q}{max} \: E_{p^*(x)}E_{p^*(y|x)}[\log q(y|x)] + 
   E_{p^*(x)}[E_{q(s,z,e,y|x)}[\frac{p^*(y|x)}{q(y|x)}\log \frac{p(x,y,s,z,e)}{q(s,z,y,e|x)}]] = \\
   &\underset{p, q}{max} \: E_{p^*(x,y)}[\log q(y|x) + E_{q(s,z,e|x)}[\frac{p(y|s,z,x)}{q(y|x)}\log \frac{p(y|s,z,x)p(x|s,z)p(z)p(s|e)p(e)}{q(y|s,z,x)q(z|x)q(s|x)q(e|x)}]] = \\
   &\underset{p, q}{max} \: E_{p^*(x,y)}[\log q(y|x) + \frac{1}{q(y|x)}E_{q(s,z,e|x)}[p(y|s,z,x)\log \frac{p(x|s,z)p(z)p(s|e)p(e)}{q(z|x)q(s|x)q(e|x)}]] = \\
   &\underset{p, q}{max} \: E_{p^*(x,y)}[\log q(y|x) + \frac{1}{q(y|x)}E_{q(s,z,e|x)}[p(y|s,z,x)\log \frac{p(x|s,z)p(z)p(s|e)p(e)p(s)}{q(z|x)q(s|x)p(s|e)p(e)}]] = \\
   &\underset{p, q}{max} \: E_{p^*(x,y)}[\log q(y|x) + \frac{1}{q(y|x)}E_{q(s,z,e|x)}[p(y|s,z,x)\log \frac{p(x|s,z)p(z)p(s)}{q(z|x)q(s|x)}]] 
\end{split}
\end{equation}
where $p(s)=\sum_e p(s|e)p(e)$ that is a GMM prior for S. $q(y|x)$ can be written as:

\begin{equation}
\label{eq:obtainq}
    q(y|x) = \int q(y|x,s,z)q(s,z|x)dsdz = \int p(y|x,s,z)q(s|x)q(z|x)dsdz = p(y|do(x)) = E_{q(s|x), q(z|x)}[p(y|x,s,z)]
\end{equation}

where $p(y|do(x))$ is the causal effect of X on Y and can be calculated by ancestral sampling. We have obtained Equation \ref{eq:finalloss} which is the objective function of GCRL. 

\section{Experiment}
In this section, we present the experimental details and more experimental results achieved when evaluating our framework. Following our proposed method in Section \ref{sec:proposed_method} of the paper, our final loss function is formulated by Equation \ref{eq:finalloss}.

\subsection{Experiment details}
The list of all hyperparameters used by our model and their corresponding settings applied when conducting our experiments are represented in Tables \ref{tab:setbench} and \ref{tab:setsyn}.

Following Tables \ref{tab:setbench} and \ref{tab:setsyn} and Equation \ref{eq:finalloss}, $S_{Dim}$ and $Z_{Dim}$ show the dimensions of the latent representation space for $S$ and $Z$ variables, respectively. 
$N^s_{q(y|x)}$ and $N^s_{E_{q(s|x),q(z|x)}}$ show the number of sampling from $S$ and $Z$ distributions for calculating the expectations of $q(y|x)$ and $E_{q(s|x),q(z|x)}$, respectively.

$R_{Input}$ shows the reconstructed input that is the \textit{relative distance} of the pedestrians from their starting locations, $d_{rel}$, which are used in our experiments conducted by the ETH-UCY dataset for a fair comparison with the state-of-the-art methods. Our hyperparameter settings ruuning experiment with ETH-UCY dataset are represented in Table \ref{tab:setbench}. Our model is trained for 300 epochs in experiments conducted with the ETH-UCY dataset.

\begin{table*}[h]
\centering
\caption{Detailed hyperparameter settings of the experiments with ETH-UCY dataset.}
\def\arraystretch{1.5}
\label{tab:setbench}
\begin{tabular}{ p{3cm}p{1.5cm}p{3cm}p{3.2cm}  }
 \textbf{Parameter} & \textbf{Setting} & \textbf{Parameter} & \textbf{Setting}\\
 \toprule
 $S_{Dim}$             & 8     & Learning Rate   & Scheduler (one-cycle) \\
 $Z_{Dim}$             & 8     & Optimizer      & Adam\\
 $N^s_{E_{q(y|x)}}$    & 10    & $n_{cluster}$ (GMM)                 & 5\\
 $N^s_{E_{q(s|x),q(z|x)}}$     & 10   & $R_{Input}$       & $d_{rel}$ \\
 \bottomrule
\end{tabular}
\end{table*}

\textbf{Note} if $N^s_{q(y|x)}=1$, then in the second term of the loss function $p(y|x,s,z)$ cancels out $q(y|x)$, therefore, all components of the second term of the loss function will not include the predicted future trajectories $y$. Considering this, for the experiments where  $N^s_{q(y|x)}=1$, we use a hyperparameter $N$, which is defined as the number of generated trajectories per sample, to improve the performance. In this case, the term $log\, q(y|x)$ in Equation \ref{eq:finalloss} becomes the variety loss \cite{SocialGAN}. We do this because our approach to causal representation learning is \textit{generative}, and therefore, we can benefit from its generative aspect to improve the performance of our model. $N$ is set to 20 during training and 100 during evaluation in experiments with synthetic dataset. Hyperparameter settings for experiments with synthetic dataset are represented in Table \ref{tab:setsyn}. The \textit{absolute locations} of the pedestrians $l_{abs}$ are used in our experiments with the synthetic dataset for a fair comparison with IM. 

We trained our method GCRL in the experiments with synthetic dataset for 250 epochs. We fine-tuned the trained models for the domain adaptation task for 100 epochs. IM has a training paradigm conducted in 4 phases with the total number of 470 epochs followed by 300 epochs for finetuning. On the other hand, our proposed method offers an end-to-end learning process which converges with much less number of epochs both for training and fine-tuning tasks.

\begin{table}[h]
\centering
\caption{Detailed hyperparameter settings of the experiments with Synthetic dataset.}
\def\arraystretch{1.5}
\begin{tabular}{ p{3cm}p{2cm} p{3cm}p{1.3cm}  }
 \textbf{Parameter} & \textbf{Setting} & \textbf{Parameter} & \textbf{Setting}\\
 \toprule
 $S_{Dim}$                     & 2    & Learning Rate                        & $5 \times 10^-3$ \\
 $Z_{Dim}$                     & 2    & Optimizer                            & Adam\\
 $N^s_{E_{q(y|x)}}$            & 1    & $n_{cluster}$ (GMM)         & 5\\
 $N^s_{E_{q(s|x),q(z|x)}}$     & 10   &  $R_{Input}$                     &  $l_{abs}$  \\
 \bottomrule
\end{tabular}\label{tab:setsyn}
\end{table}

\newpage
\subsection{Additional Experiments}
\subsubsection{Motion forecasting with ETH-UCY}
We present motion forecasting results of all environments of the ETH-UCY dataset; 'ETH', 'HOTEL', 'UNIV', 'ZARA1' and 'ZARA2'. These experiments are conducted without the controlled spurious correlations. Table \ref{tab:allenvs} presents a comparison between performances of GCRL and the base model STGAT \cite{STGAT} with similar settings.

\begin{table}[h]
	\centering
	\caption{Motion forecasting results on different domains of ETH-UCY}
	\def\arraystretch{1.6}
	\label{tab:allenvs}
	\begin{tabular}{lcccccc}
		\textbf{Model} & \multicolumn{6}{c}{\textbf{ADE/FDE}}\\
		\toprule 
		& ETH & HOTEL & UNIV & ZARA1 & ZARA2 & \textbf{AVG} \\
		\toprule
		STGAT &   0.8/1.53 & 0.52/1.08 & 0.51/1.12 & 0.39/0.87 & 0.3/0.64 & 0.5/1.05 \\
		GCRL  &   0.97/1.9 &    0.55/1.14       &     0.51/1.13      &     0.38/0.84      &    0.3/0.67      &       0.54/1.1   \\
		\bottomrule
	\end{tabular}
\end{table}

Figures \ref{fig:mdleftup} and \ref{fig:mdrightup} show the ADE while training and validating our model on the 'hotel', 'univ', 'zara1', and 'zara2' environments of the ETH-UCY dataset for 300 epochs, respectively. The trained model is tested on the 'eth' environment at inference. Figures \ref{fig:mdleftdown} and \ref{fig:mdrightdown}, on the other hand, show the FDE for training and validation, respectively. The convergence of prediction and reconstruction losses during training are shown in Figures \ref{fig:mdleftup2} and \ref{fig:mdrightup2},  respectively. Large values of reconstruction and prediction losses are due to the constant terms added to $p(x|,s,z)$ and $p(y|x,s,z)$ such as the log of determinant of covariance matrices. The regularization loss of $S$ and $Z$ distributions are also represented by Figures \ref{fig:mdleftdown2} and \ref{fig:mdrightdown2}, respectively. As shown in the Figures, the regularization losses converge to zero showing that the model drives the posterior distributions of $S$ and $Z$ towards their priors.

\begin{figure}[!h]
	\centering
	\subfloat[Training]{\label{fig:mdleftup}{\includegraphics[width=0.3\textwidth]{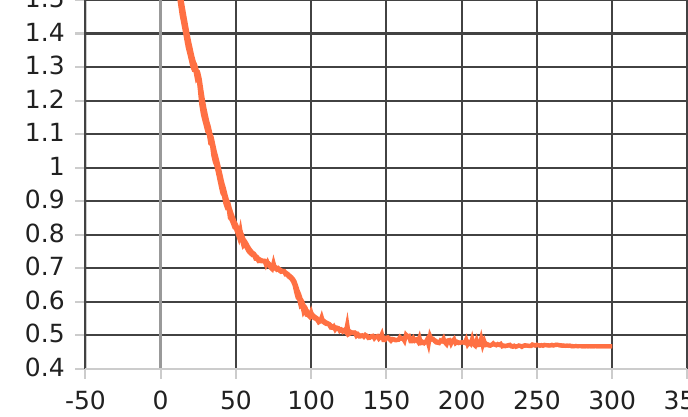}}}
	\subfloat[Validation]{\label{fig:mdrightup}{\includegraphics[width=0.3\textwidth]{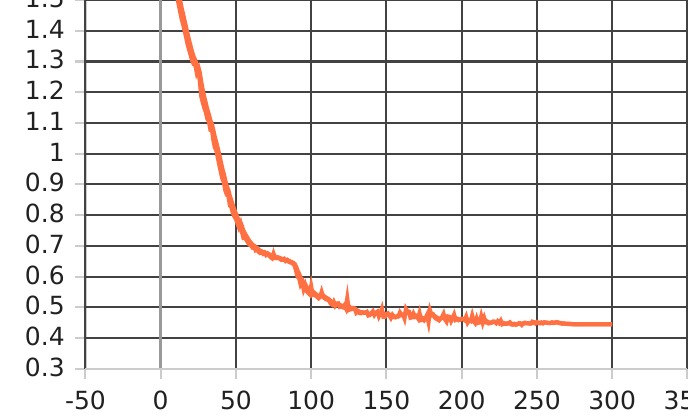}}}
	\caption{ADE}
	\subfloat[Training]{\label{fig:mdleftdown}{\includegraphics[width=0.3\textwidth]{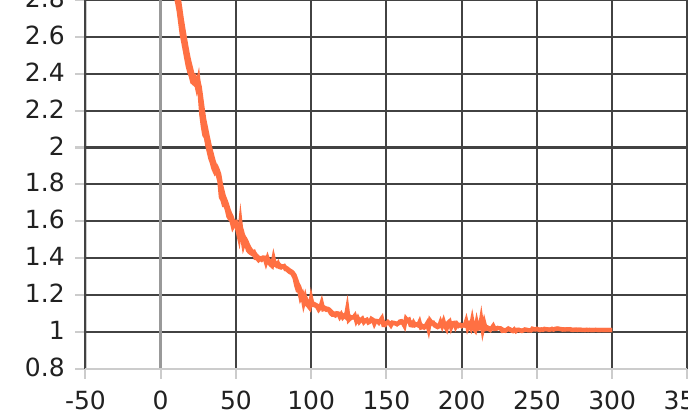}}}
	\subfloat[Validation]{\label{fig:mdrightdown}{\includegraphics[width=0.3\textwidth]{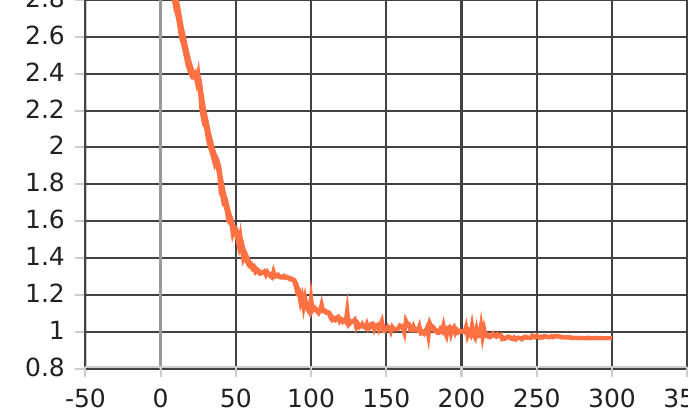}}}
	\caption{FDE}
	\label{fig:subfigures}
\end{figure}

\begin{figure}[!h]
	\centering
	\subfloat[Prediction Loss]{\label{fig:mdleftup2}{\includegraphics[width=0.3\textwidth]{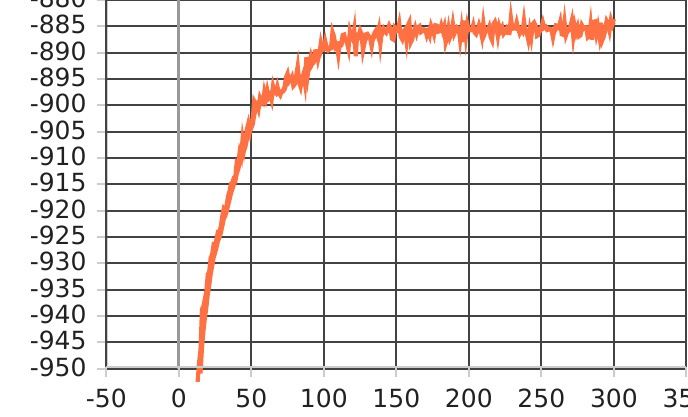}}}
	\subfloat[Reconstruction Loss]{\label{fig:mdrightup2}{\includegraphics[width=0.3\textwidth]{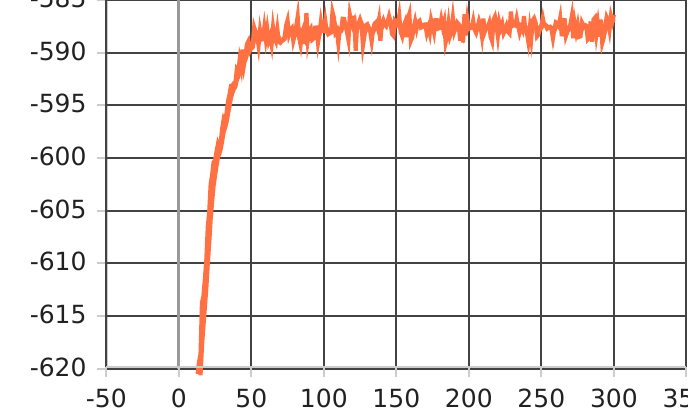}}}
	\caption{Predictions}
	\subfloat[$S$ Regularization Loss]{\label{fig:mdleftdown2}{\includegraphics[width=0.3\textwidth]{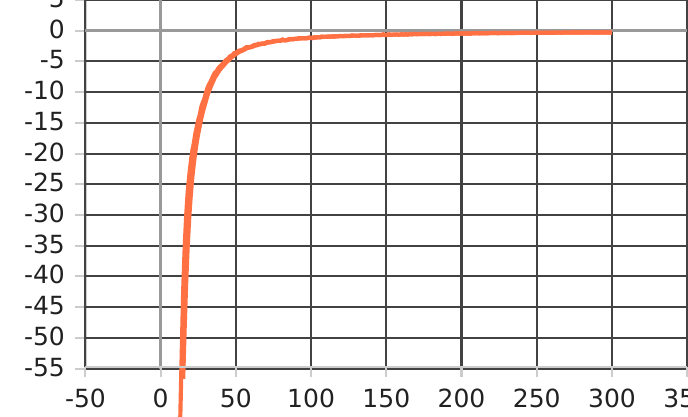}}}
	\subfloat[$Z$ Regularization Loss]{\label{fig:mdrightdown2}{\includegraphics[width=0.3\textwidth]{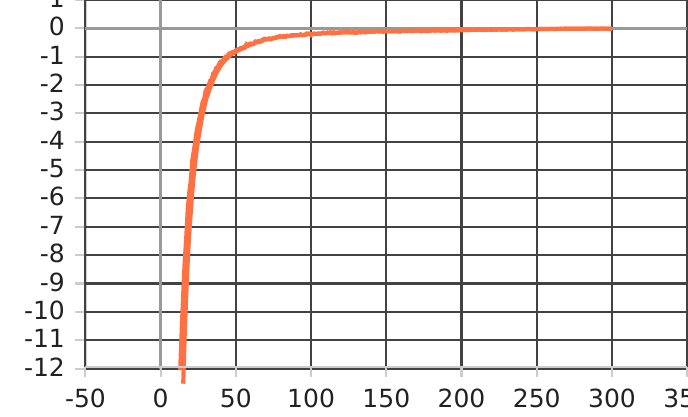}}}
	\caption{Regularizations}
	\label{fig:subfigures}
\end{figure}

\subsubsection{Ablation study on ETH-UCY}
In Table \ref{tab:abl_bench}, the results for Ablation studies are presented. We used different settings to evaluate our method in these experiments including: removing coupling layers, changing the dimensions of $S$ and $Z$ latent representations, and applying different number of clusters for GMM used to model prior of $S$. 

\begin{table}[!h]
	\centering
	\caption{The effect of different components on performance when experimenting with the ETH-UCY dataset.}
	\def\arraystretch{1.4}
	\label{tab:abl_bench}
	\vskip 0.15in
	\begin{center}
		\begin{small}
			\begin{sc}
				\begin{tabular}{lcc}
					\textbf{Models} & \textbf{ADE} & \textbf{FDE} \\
					\toprule
					Vanila GCRL &  0.9728 & 1.8696  \\
					
					No coupling layers & 0.9881 &  1.8869 \\
					
					latent dimension:2 & 0.9736 & 1.8459  \\
					
					Number of GMM clusters:2 & 0.9732 & 1.9534  \\
					
					Number of GMM clusters:7 &  0.9775 & 1.9169   \\
					\bottomrule
				\end{tabular}
			\end{sc}
		\end{small}
	\end{center}
	\vskip -0.1in
\end{table}

\subsubsection{Qualitative results} Figure \ref{fig:multi_vis} visualizes the capability of GCRL in generating multiple trajectories per pedestrian. 

\begin{figure}[!h]
	\centering
	\subfloat[Scene 1]{\includegraphics[width=0.5\textwidth]{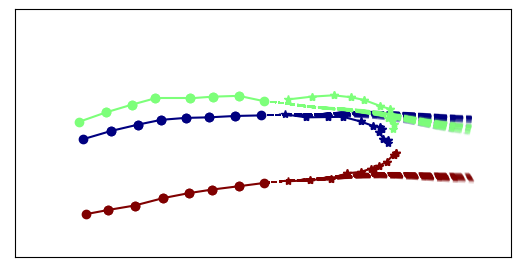}}
	\subfloat[Scene 2]{\includegraphics[width=0.5\textwidth]{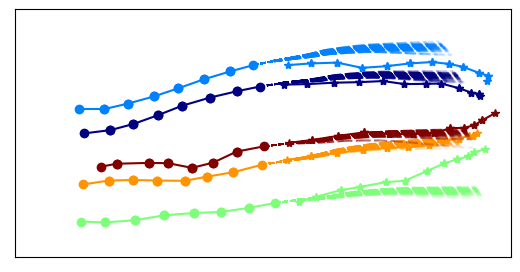}}
	\hfill
	\subfloat[Scene 3]{\includegraphics[width=0.5\textwidth]{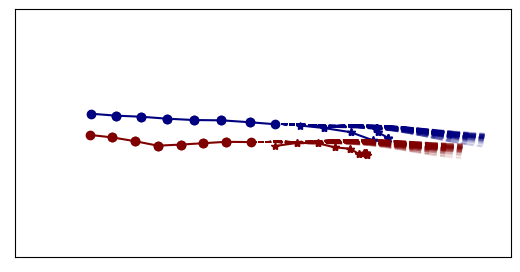}}
	\subfloat[Scene 4]{\includegraphics[width=0.5\textwidth]{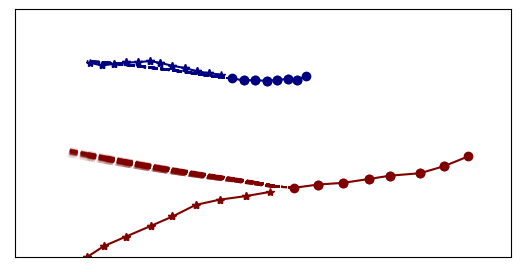}}
	\caption{Trajectories of pedestrians in different scenes. The circled trajectories are observed, the dashed trajectories are predicted, and the starred trajectories are groundtruth. In all of the scenes, 100 trajectories are predicted for each pedestrian.}
	\label{fig:multi_vis}
\end{figure}

\subsubsection{Changing multiple factors simultaneously} In this experiment, we change the environments in the ETH-UCY dataset and noise intensities simultaneously. The results are shown in Table \ref{tab:envnoise}.

\begin{table}[!h]
	\centering
	\caption{GCRL performance in cases with multiple changing factors.}
	\def\arraystretch{1.4}
	\label{tab:envnoise}
	\begin{tabular}{lcccccc}
		\centering
		\textbf{Environment} & \multicolumn{4}{c}{\textbf{ADE/FDE}}\\
		\toprule 
		& $\alpha=8$ & $\alpha=16$ & $\alpha=32$ & $\alpha=64$  \\
		\toprule
		ETH &   0.97/1.8 & 0.97/1.8 & 0.97/1.8 & 0.97/1.8  \\
		
		HOTEL  & 0.64/1.3  & 0.64/1.3  & 0.64/1.3   &  0.64/1.3   \\
		
		UNIV  & 0.53/1.15  & 0.53/1.15  & 0.53/1.15   &  0.53/1.15   \\
		
		ZARA1  & 0.4/0.86  & 0.4/0.86  & 0.4/0.86   &  0.4/0.86   \\
		
		ZARA2  & 0.31/0.66 & 0.31/0.66 & 0.31/0.66  &  0.31/0.66  \\
		\bottomrule
	\end{tabular}
\end{table}

\end{document}